\documentclass[fleqn,10pt,twocolumn]{SICE_FES25}

\usepackage{amsmath}
\usepackage{amssymb}
\usepackage{color}
\usepackage{enumerate}
\usepackage{graphicx}
\usepackage{subcaption}

\usepackage{algorithm}
\usepackage{algpseudocode}

\newcommand{\R}{{\mathbb{R}}}

\renewcommand{\top}{\mathsf{T}}

\title{Learning the Simplest Neural ODE}

\author{Yuji Okamoto*${}^{1\dagger}$, Tomoya Takeuchi*${}^{2}$, and Yusuke Sakemi*${}^{2}$}
\speaker{Yuji Okamoto}

\affils{${}^{1}$Gradiate School of Medicine, Kyoto University, Kyoto, Japan\\
(E-mail: okamoto.yuji.2c@kyoto-u.ac.jp)\\
${}^{2}$Research Center for Mathematical Engineering, Chiba Institute of Technology, Chiba, Japan\\
(E-mail: \{yusuke.sakemi, tomoya.takeuchi\}@p.chibakoudai.jp)\\
}

\abstract{%
Since the advent of the ``Neural Ordinary Differential Equation (Neural ODE)'' paper\cite{chen2018neural}, learning ODEs with deep learning has been applied to system identification, time-series forecasting, and related areas. 
Exploiting the diffeomorphic nature of ODE solution maps, neural ODEs have also enabled their use in generative modeling. 
Despite the rich potential to incorporate various kinds of physical information, training Neural ODEs remains challenging in practice. 
This study demonstrates, through the simplest one-dimensional linear model, why training Neural ODEs is difficult.
We then propose a new stabilization method and provide an analytical convergence analysis. 
The insights and techniques presented here serve as a concise tutorial for researchers beginning work on Neural ODEs.
}


\begin{document}

\maketitle
\footnote[0]{* Equal contributions.}

\section{Introduction}
The learning of ordinary differential equations (ODEs)
\begin{align}
\hspace{2em}
\dot{x} = f_\theta (x),\quad x(0) = x_0, \label{Eq:ODE}
\end{align}
where the dynamics \(f_\theta: \R^n \rightarrow \R^n\) are modeled by a neural network (NN), is widely used for modeling and predicting a variety of physical systems such as robotics and electronic circuits\cite{zhao2024neural,xiong2022neural}.
Because ODE models allow one to embed inherent properties---for example, Hamiltonian structure \cite{greydanus2019hamiltonian} or stability guarantees \cite{Manek2019,kojima2022learning,okamoto2025learning}---Neural ODEs often achieve accurate long-term forecasting.
Moreover, treating the solution map of an ODE as a diffeomorphism has led to applications in generative modeling. Compared to normalizing flows \cite{rezende2015variational}, such continuous-time models can yield low-parameter, memory-efficient generators \cite{finlay2020train}.

Gradient-based optimization is the de-facto standard for learning NN-defined dynamics, enabled by the adjoint method \cite{chen2018neural} for efficient gradient computation. Nonetheless, several issues arise:
\begin{enumerate}[(i)]
    \item Sensitivity of gradients to the sampling interval of time-series data,
    \item Variation of convergence speed depending on NN initialization,
    \item Difficulty of tuning hyper-parameters such as leaning rate.
\end{enumerate}
These factors often cause training instability or gradient vanishing/explosion. Existing work addresses them heuristically by restricting the search space of dynamics \cite{finlay2020train} or by tailored initialization \cite{westny2024stability}. However, the fundamental question---\textbf{\textit{ why is training Neural ODE hard?}}---remains unanswered.

To this end, we clarify the difficulty using the simplest ODE and propose a new gradient rule that mitigates it. Our contributions are:
\begin{enumerate}[1]
    \item We reveal, via a one-dimensional linear system, why standard gradient descent struggles.
    \item We introduce a new learning rule leveraging the variance of terminal states.
    \item We demonstrate its effectiveness on a toy example.
\end{enumerate}

\begin{algorithm*}[t]
\caption{Gradient via Adjoint Method}
\begin{algorithmic}
\Require Dynamics $f_{\theta}$, time $t$, initial state $x_0$, terminal state $x_t$, true state $x_t^* = \phi_{f^*}^t(x_0)$
\Ensure $ \tfrac{\partial {\mathcal{L}}}{\partial \theta }  \triangleq \tfrac{\partial}{\partial \theta } \Big(\tfrac{1}{2}\|\phi_{f_\theta }^t(x_0)  - x_t^*\|^2 \Big)$
\vspace{2pt}
\State $(x_t,y_t,z_t) \gets \left(x_t,\tfrac{\partial{\mathcal{L}}}{\partial x_t}, {\bf 0} \right)$ \Comment{Initialize augmented state} \Statex
\State $(x_0,y_0,z_0) \gets \texttt{ODESolver} \bigg( \Big(-f_\theta,-\tfrac{\partial f_{\theta}}{\partial x}^{\top} y, - y^{\top} \tfrac{\partial f_{\theta}}{\partial \theta} \Big),(x_t,y_t,z_t),t\bigg)$ 
\Comment{Solve Reversal-time ODE}\\
 \hfill\text{/* $\texttt{ODESolver}(\text{dynamics},\text{initial stete},\text{time})$ \textit{is an numerical ODE integrator */}}\Statex
\Return $z_{0}$ \Comment{Return gradient}
\end{algorithmic}
\label{Alg:adjoint_method}
\end{algorithm*}

\section{Related Work}
Prior studies have shown that constraints on the dynamics, on the solution map, and on parameter initialization are all crucial for training Neural ODEs.

\textbf{Dynamics-constrained Neural ODEs:}
Manek and Kolter \cite{Manek2019} proposed constraining the learned dynamics to be strictly Lyapunov-stable with respect to a learned Lyapunov function, thereby improving accuracy for complex stable systems. Stability constraints have been extended to input–output systems, leading to input–output stable \cite{kojima2022learning} and dissipative Neural ODEs \cite{okamoto2025learning}, both of which exhibited performance gains.

\textbf{Solution-map-constrained Neural ODEs:}
Treating the ODE solution map as a continuous normalizing flow, FFJORD \cite{grathwohl2018scalable} introduced regularization for probabilistic change-of-variables, enabling continuous conditional density estimation. Subsequent work improved learning rules from an optimal-transport viewpoint and tackled computational costs \cite{finlay2020train}.

\textbf{Initialization techniques:}
Westny \emph{et al.} \cite{westny2024stability} proposed initializing parameters inside the stability region of the chosen ODE solver, observing that standard initializations sometimes yield eigenvalues outside this region and degrade performance, thereby improving both convergence and accuracy.

\section{Background}

Training a Neural ODE is typically formulated as minimizing the expected terminal error. Let \(f_\theta: \R^n \rightarrow \R^n\) denote the learned dynamics and \(f_*:\R^n\rightarrow \R^n\) the true dynamics generating the data. Denote the ODE solution at time \(t\) by \(\phi_f^t(x)\). The terminal loss at time \(t\) is
\begin{align}
\hspace{2em}    {\mathrm{Loss}}(\theta) \triangleq  {\mathrm{E}}_{x\sim p} \Bigg[\frac{1}{2}\big\|\phi_{f_\theta}^t(x) - \phi_{f_*}^t(x)\big\|^2\Bigg],\label{Eq:Loss_func}
\end{align}
where \(p\) is the distribution of initial state.

Parameters are updated via gradient descent:
\begin{align}
\hspace{2em}    \theta_{k+1} =  \theta_k  - \eta \nabla {\mathrm{Loss}}(\theta_k),\label{Eq:gradient_method}
\end{align}
with learning rate \(\eta>0\).
In Neural ODEs, the gradient of $\mathrm{Loss}$ is efficiently computed via the adjoint method, enabling the learning of complex nonlinear dynamics represented by neural networks. {\bf Algorithm~\ref{Alg:adjoint_method}} presents the corresponding pseudocode.

Although gradient descent enjoys linear convergence on smooth, strongly convex functions, such assumptions rarely hold in deep learning, where the parameter landscape is high-dimensional and highly non-convex. Consequently, practitioners rely on optimizers and hyper-parameter tuning without analyzing loss properties.

Importantly, \textbf{\textit{ the difficulty in training Neural ODEs does not stem from the complexity of \(f_\theta\) itself}}. In the next section, we use the simplest one-dimensional linear dynamics to pinpoint the cause of optimization difficulty.

\section{Problem}

\begin{figure}[t]
    \centering
    \begin{tabular}{c}
    \begin{subfigure}[b]{0.38\textwidth}
        \includegraphics[width=\textwidth]{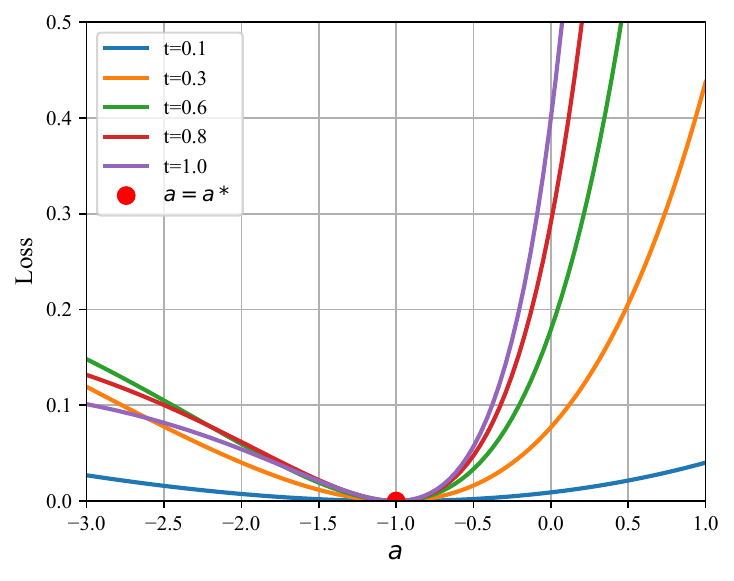}
    \end{subfigure}\\
        \text{(a) ${\mathrm{Loss}}(a)$ for varying $t$ ($a^*=-1$)}\\
    \vspace{0.1cm}\\
    \begin{subfigure}[b]{0.38\textwidth}
        \includegraphics[width=\textwidth]{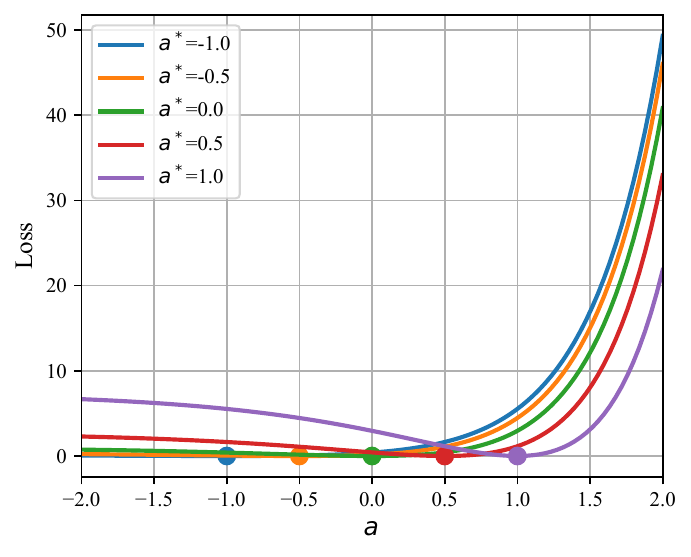}
    \end{subfigure}\\
         \text{(b) ${\mathrm{Loss}}(a)$ for varying $a^*$ ($t=1$)}
   \end{tabular}
    \caption{Visualization of ${\mathrm{Loss}}(a)$ under different settings. 
    (a) shows the effect of time $t$ with fixed $a^*=-1$, while 
    (b) illustrates the influence of different target values $a^*$ with fixed $t=1$. Markers indicate the minima where $a = a^*$.}
    \label{fig:loss_functions}
\end{figure}

\begin{algorithm*}[t]
\caption{New Update role of Neural ODE}
\begin{algorithmic}
\Require Initial parameter $\theta_0$, time $t$, dataset ${\mathcal{D}}\triangleq \{x_0^{(i)},x_t^{*(i)}= \phi_{f_*}^t(x_0^{(i)})\}_{i=1}^N$
\Ensure Optimized parameter $\theta$
\vspace{1pt}
\State $\theta \gets \theta_0$
\While{\textbf{not} converged}
    \State ${\mathcal{G}},{\mathcal{X}}_t \gets \{\},\{\}$ \Comment{Initialize gradient and terminal state set}
    \ForAll{ $(x_0,x_t^*)\in\mathcal{D}$}
        \State $x_t  \gets \texttt{ODESolve} \big( f_{\theta},x_0,t\big)$\Comment{Estimate terminal state}
        \State $g \gets$ {\bf Algorithm~\ref{Alg:adjoint_method}} with $(f_\theta,t,x_0, x_t,x_t^*)$\Comment{Calculate Gradient via adjoint method}
        \State ${\mathcal{X}}_t.\text{append}(x_t)\quad {\mathcal{G}}.\text{append}(g)$ \Comment{Save each terminal state and gradient}
    \EndFor
    \State $\sigma^2_t\gets \tfrac{1}{N}\sum_{x_t \in {\mathcal{X}}_t}x_t^2$ \Comment{Calculate sample terminal variance}
    \State $\nabla{\rm Loss}\gets \tfrac{1}{N}\sum_{g \in {\mathcal{G}}}g$ \Comment{Calculate sample mean gradient}
    \State $\theta\gets \theta - \eta \tfrac{1}{t^2\sigma^2_t}\nabla{\rm Loss}$ \Comment{Update parameters}
\EndWhile
\end{algorithmic}
\label{Alg:update_role}
\end{algorithm*}

For simplicity, we analyze one-dimensional linear dynamics. Let the learned and true dynamics be
\begin{align}
\hspace{2em}     f_\theta(x) \triangleq  a x,  \quad f_*(x) \triangleq a^* x\label{Eq:1dim_dynamics}
\end{align}
and assume the initial distribution \(p\) is Gaussian with mean \(0\) and variance \(\sigma^2\).
Then the loss \eqref{Eq:Loss_func} becomes
\begin{align}
\hspace{2em}    {\mathrm{Loss}}(a) =  \frac{\sigma^2}{2}(e^{a t} - e^{a^*t})^2. \label{Eq:Loss_1dim}
\end{align}

This loss is neither strongly convex on \(\R\) nor even convex: it is concave for \(a < a^* - \tfrac{\ln 2}{t}\).
Figure~\ref{fig:loss_functions}(a) shows that increasing the terminal time \(t\) accentuates asymmetry, while Figure~\ref{fig:loss_functions}(b) confirms the asymmetry regardless of \(a^*\).

\textbf{\textit{Asymmetry of the terminal loss function explains the difficulty of the gradient method:}} selecting \(\eta\) suitable for \(a>a^*\) causes gradient vanishing, whereas tuning for \(a<a^*\) leads to gradient explosion. Both initialization and learning rate must therefore be carefully co-tuned, which is highly sensitive, as the scale in Fig.~\ref{fig:loss_functions}(b) reveals.

We next quantify the dependence of convergence speed on these factors.

\begin{lemma}\label{lem:loss_step}
If the learning rate \(\eta\) is sufficiently small, then
\begin{align*}
\hspace{2em}    {\mathrm{Loss}}(a_{k+1}) &=  \bigl(1- 2\eta t^2 \sigma_k^2(t)\bigr) {\mathrm{Loss}}(a_{k}),
\end{align*}
where the terminal-state variance is
\begin{align*}
\hspace{2em}    \sigma_k^2(t) \triangleq {\mathrm{E}}\Bigl[ \bigl(\phi_{f_{\theta_k}}^t(x)\bigr)^2\Bigr] = \sigma^2e^{2a_kt}.
\end{align*}
\end{lemma}
\begin{proof}
Let \(\Delta a = -\eta {\mathrm{Loss}}^\prime(a_k)\). Then
\begin{align*}
\hspace{2em}    &{\mathrm{Loss}}(a_{k+1}) \\
    &={\mathrm{Loss}}(a_{k}+\Delta a)\\
    &={\mathrm{Loss}}(a_{k})+{\mathrm{Loss}}^\prime(a_k)(\Delta a)+{\mathcal{O}}(\eta^2).
\end{align*}
Because
\begin{align*}
\hspace{2em}    {\mathrm{Loss}}^\prime(a_k)(\Delta a)&=-\eta |{\mathrm{Loss}}^\prime(a_k)|^2
   \\
   &=-\eta |t e^{a_kt} (e^{a_kt} - e^{a^*t})\sigma^2|^2\\
 &=-2\eta t^2 \sigma^2 e^{2a_kt}  {\mathrm{Loss}}(a_{k}),
\end{align*}
the claim follows.
\end{proof}

Hence, convergence speed depends on the evolving terminal variance \(\sigma_k^2(t)\), which changes with initialization and during training, making optimization difficult. Yet, this variance is observable during training, suggesting a way to incorporate it into the update rule. The next section develops such a method.

\section{Method}

\begin{figure*}[t]
    \centering
        \begin{tabular}{ccc}
    \begin{subfigure}[b]{0.295\textwidth}
        \includegraphics[width=\textwidth]{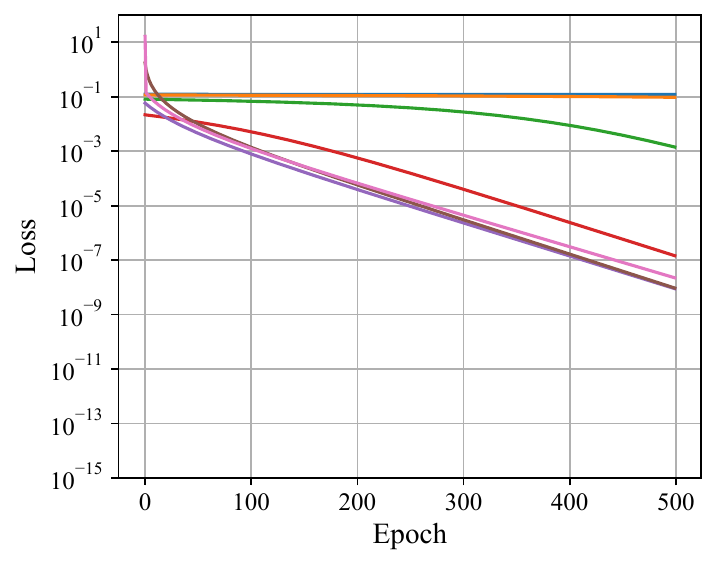}
    \end{subfigure}&
    \begin{subfigure}[b]{0.295\textwidth}
        \includegraphics[width=\textwidth]{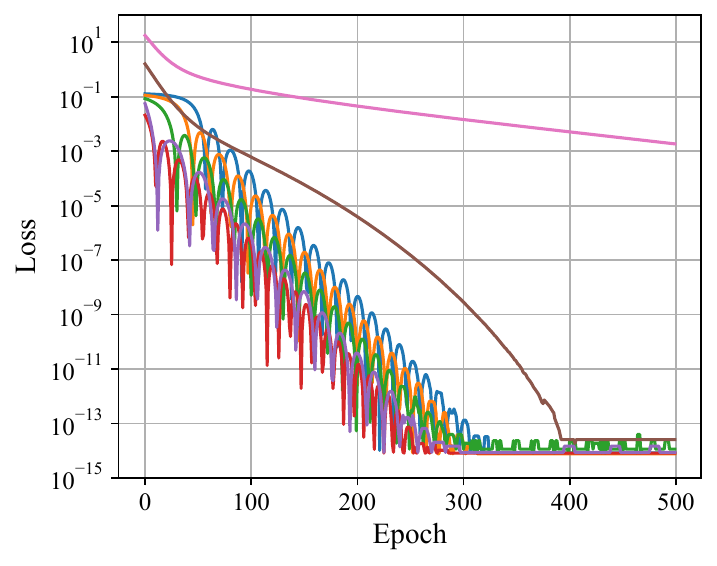}
    \end{subfigure}&
    \begin{subfigure}[b]{0.295\textwidth}
        \includegraphics[width=\textwidth]{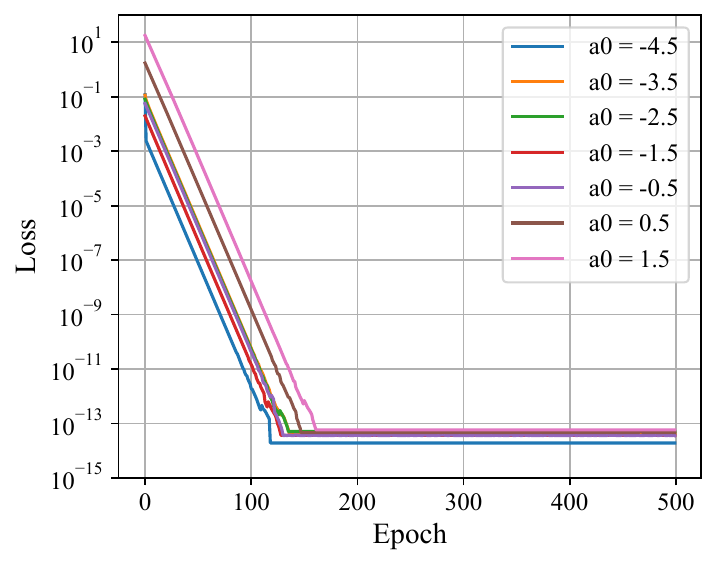}
    \end{subfigure}\\
    \begin{subfigure}[b]{0.295\textwidth}
        \includegraphics[width=\textwidth]{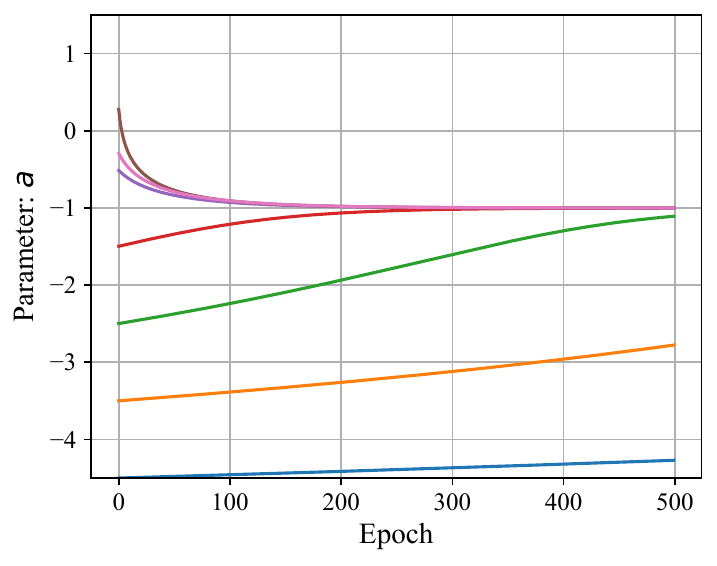}
        \caption{SGD}
    \end{subfigure}&    
    \begin{subfigure}[b]{0.295\textwidth}
        \includegraphics[width=\textwidth]{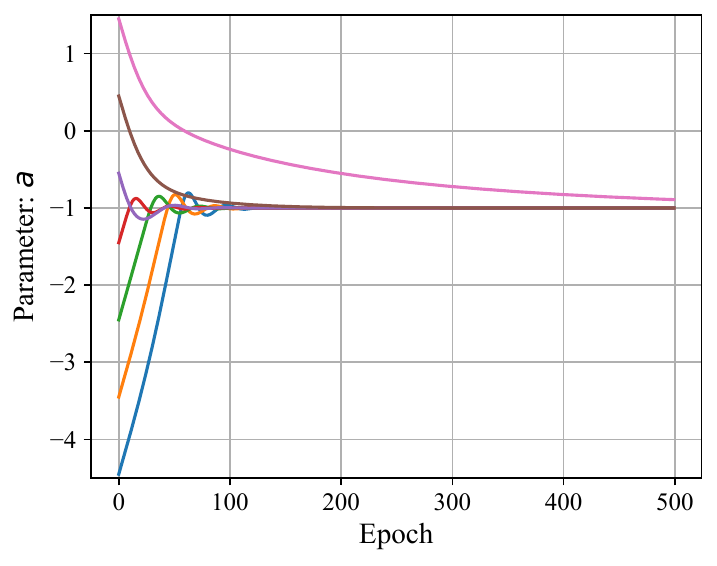}
        \caption{Adam}
    \end{subfigure}&
    \begin{subfigure}[b]{0.295\textwidth}
        \includegraphics[width=\textwidth]{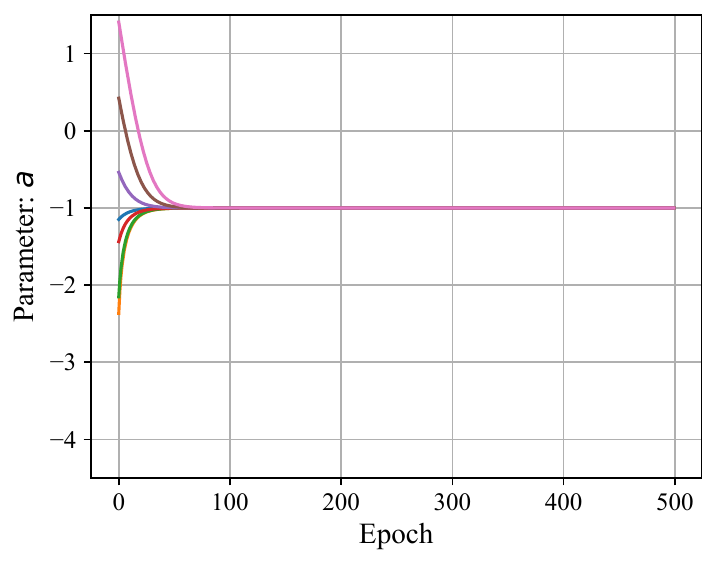}
        \caption{Natural Gradient (Theorem~1)}
    \end{subfigure}
    \end{tabular}
    \caption{Result of each gradient method. The color of the lines represents the difference in the initial values.}
    \label{fig:result}
\end{figure*}

We propose a new gradient update that achieves linear convergence for the simplest ODE.

\begin{theorem}\label{thm:main}
For the dynamics in \eqref{Eq:1dim_dynamics}, define the variance-corrected update
\begin{align*}
\hspace{2em}    a_{k+1} = a_k - \eta \frac{1}{t^2\sigma^2_k(t)} \nabla {\mathrm{Loss}}(a_k).
\end{align*}
If the learning rate \(\eta>0\) is sufficiently small, then
\begin{align*}
\hspace{2em}    {\mathrm{Loss}}(a_{k}) = (1-2\eta)^k{\mathrm{Loss}}(a_0).\\
\end{align*}
\end{theorem}

\begin{proof}
Directly follows from Lemma~\ref{lem:loss_step}.
\end{proof}

Whereas standard gradient descent depends on \(t\) and \(\sigma_k^2(t)\), our correction removes this dependence, guaranteeing linear convergence of the loss to zero.
{\bf Algorithm~\ref{Alg:update_role}} provides the pseudocode for the aforementioned gradient method.
One can see that \(\sigma_k^2(t)\) is obtained naturally when deriving the gradient with the adjoint method.

\section{Experiments}

To verify effectiveness, we trained the one-dimensional linear dynamics \(f_\theta(x)=ax\) toward the target \(f_*(x)=-x\) with learning rate \(\eta=0.05\). Baselines were stochastic gradient descent (SGD) and Adam \cite{kingma2014adam} with the same learning rate.

Figure~\ref{fig:result} shows the outcomes. Our method (right column) achieves linear loss decay for all initializations. In contrast, SGD and Adam (left and center) exhibit initialization-dependent behavior and may fail to converge linearly. The oscillations under Adam indicate that its adaptive step-size scheme, designed for gradients of roughly quadratic scale, is ill-suited for the exponential scaling inherent in ODE training, where terminal variance grows exponentially with gradients.

\section{Discussion}

Although we addressed the simplest dynamics, extending to nonlinear high-dimensional systems raises two issues.

\textbf{Theorem~\ref{thm:main} as the Natural Gradient Method:}

So, what was the correction term proposed in our Theorem?
This actually corresponds to the inverse of the Fisher Information Matrix:
\begin{align*}
F(a) &\triangleq {\mathrm{E}}
\left[ \left( \nabla \phi_{f_\theta}^t(x)\right) \left( \nabla \phi_{f_\theta}^t(x)\right)^{\top} \right] \\
&= {\mathrm{E}}\left[(te^{at} x)^2 \right] = t^2\sigma^2_k(t)
\end{align*}
Therefore, Theorem~\ref{thm:main} reduces to nothing more than the natural‑gradient method~\cite{amari1998natural} for the simplest ODE.
This suggests that \textbf{\textit{training a Neural ODE appears to be a problem on which the only natural‑gradient method can be exceptionally efficient}}.

On the other hand, applying the natural gradient method to the dynamics $f_\theta$ of a general neural network is difficult.
The computation of the inverse of the Fisher Information Matrix requires a computational complexity of at least $O(n^2)$, where $n$ is the number of parameters.
Consequently, efficiently training neural networks for deep learning, which involves a vast number of parameters, remains a challenge.

\textbf{Non-uniqueness of dynamics for a given endpoint map:}

Even with Lipschitz continuity, the dynamics \(f\) are not uniquely determined by the solution map \(\phi^t_f:\R^n\to\R^n\). Linear multi-dimensional systems possess rotational degrees of freedom; nonlinear systems admit continuous freedom in velocity fields, complicating optimization. Uniqueness requires additional conditions, such as temporal sampling or structural constraints.


\section{Conclusion}
Using a one-dimensional linear system, we (1) elucidated the source of training difficulty in Neural ODEs, (2) proposed a variance-corrected gradient rule with proven linear convergence, and (3) validated the theory on a toy example. 
Furthermore, we confirm that this gradient-based method is equivalent to the natural gradient method for Neural ODEs.
Although limited to a trivial setting, the insights serve as a stepping stone for future Neural ODE research.

\section*{Acknowledgements}
This work was supported by JST Moonshot R\&D Grant Number JPMJMS2021, JST PRESTO Grant Number JPMJPR22C5, and JSPS KAKENHI Grant Number 25K00148
{\tiny
\bibliographystyle{unsrt}
\bibliography{reference}
}
\end{document}